\documentclass[runningheads]{llncs}
\usepackage[english]{babel}

\usepackage{xcolor}
\usepackage[T1]{fontenc}
\usepackage{amsmath}      
\usepackage{amsfonts}
\usepackage{algorithm}
\usepackage{algorithmic}
\usepackage{booktabs}
\usepackage{multirow}
\usepackage{makecell}
\usepackage{hyperref}
\usepackage{graphicx}
\begin{document}
\title{Concisely Explaining the Doubt: Minimum-Size Abductive Explanations for Linear Models with a Reject Option}
\titlerunning{Minimum-Size AXps for Linear Models with a Reject Option}
%
\author{Gleilson Pedro Fernandes \and
Thiago Alves Rocha} 
%
\authorrunning{Fernandes and Rocha}
%
\institute{Instituto Federal do Ceará (IFCE), Brazil \\
\email{gleilson@aluno.ifce.edu.br} and \email{thiago.alves@ifce.edu.br}}
\maketitle              
\begin{abstract}
Trustworthiness in artificial intelligence depends not only on what a model decides, but also on how it handles and explains cases in which a reliable decision cannot be made. In critical domains such as healthcare and finance, a reject option allows the model to abstain when evidence is insufficient, making it essential to explain why an instance is rejected in order to support informed human intervention. In these settings, explanations must
not only be interpretable, but also faithful to the underlying model and computationally efficient enough to support real-time decision making. Abductive explanations guarantee fidelity, but their exact computation is known to be NP-hard for many classes of models, limiting their practical applicability. Computing \textbf{minimum-size} abductive explanations is an even more challenging problem, as it requires reasoning not only about fidelity but also about optimality. Prior work has addressed this challenge in restricted settings, including log-linear-time algorithms for computing minimum-size abductive explanations in linear models without rejection, as well as a polynomial-time method based on linear programming for computing abductive explanations, without guarantees of minimum size, for linear models with a reject option. In this work, we bridge these lines of research by computing minimum-size abductive explanations for linear models with a reject option. For accepted instances, we adapt the log-linear algorithm to efficiently compute optimal explanations. For rejected instances, we formulate a 0–1 integer linear programming problem that characterizes minimum-size abductive explanations of rejection. Although this formulation is NP-hard in theory, our experimental results show that it is consistently more efficient in practice than the linear-programming-based approach that does not guarantee minimum-size explanations.

\end{abstract}

\section{Introduction}

Machine learning systems have been increasingly incorporated into decision-making processes that directly affect people's lives, such as medical diagnosis, credit approval, and applications in the justice system. In these scenarios, it is not sufficient for a model to achieve high predictive accuracy. It is equally fundamental that its decisions can be understood, audited, and considered trustworthy by human users and regulatory bodies. This challenge has driven advances in Explainable Artificial Intelligence (XAI), whose central goal is to make model behavior more transparent and justifiable.

However, explanations are only truly useful in practice when they combine three essential properties: interpretability, fidelity to the underlying model, and computational efficiency. Methods that sacrifice fidelity in favor of speed may induce users into a false sense of confidence, especially in high-risk domains \cite{lipton2018mythos,rudin2019stop,slack2020fooling}. 

In this context, abductive explanations, grounded in formal logic, have received significant attention \cite{ignatiev2019abduction,marquessilva2022logic}. Such explanations identify non-redundant subsets of features that are sufficient to support a model decision, regardless of the values taken by the remaining features, thereby offering strong guarantees of fidelity \cite{ignatiev2019abduction,darwiche2020reasons}. However, these explanations are not necessarily of minimum size, which can hinder interpretability, as smaller explanations are easier for human decision-makers to inspect, communicate, and act upon. Furthermore, obtaining such explanations typically requires exploring an exponential combinatorial space, making exact approaches impractical for medium- and large-scale models, particularly in the case of deep neural networks with high-dimensional inputs \cite{ignatiev2019abduction}.

Recent work \cite{marquessilva2020,barcelo2020model} have shown that, for specific classes of models, it is possible to circumvent these limitations by exploiting structural properties. In particular, linear classifiers allow the efficient derivation of minimum-size abductive explanations in log-linear-time, by leveraging the bounded domain of the input features \cite{marquessilva2020}.

Most existing explanation methods implicitly treat all predictions in the same way, regardless of the level of uncertainty associated with them. This becomes problematic in scenarios where the model operates under high uncertainty, which are common in high-stakes domains. To address this issue, a \emph{reject option} allows models to abstain when the available evidence is insufficient to support a reliable conclusion \cite{chow1970,geifman2019selectivenet,hendrickx2024machine}.

While the reject option improves decision safety by preventing unreliable predictions, it also raises a fundamental explanatory question: ``\emph{why} was a given instance rejected?'' \cite{artelt2023idontknow}. From an XAI perspective, rejection does not eliminate the need for explanations; rather, it changes their role. Instead of justifying a specific predicted outcome, explanations for rejected instances must clarify the conditions under which the model is unable to reach a reliable conclusion, so that human experts can understand the source of uncertainty and decide how to proceed \cite{rochafilho2023,punzi2024l2lore}. In this setting, minimum-size abductive explanations become particularly important. Explanations based on a small number of features more clearly expose the sources of uncertainty in the model’s decision process, facilitating efficient inspection and informed intervention by human experts.

Linear models with a reject option could also be particularly valuable as
\emph{surrogate} explainers for complex predictors. Instances falling within the surrogate model’s rejection region might correspond to areas of high uncertainty in the original predictor. Rather than forcing potentially misleading explanations, rejection may signal the limits of the surrogate’s reliability. In this setting, minimum-size abductive explanations computed for rejected instances can help expose a smallest set of features that is sufficient to characterize
this uncertainty, supporting efficient inspection and informed human
intervention. We leave for future work the exploration of minimum-size abductive explanation techniques using linear surrogate models with a reject option for explaining complex predictors.

In light of the considerations above, we propose \textbf{MINABRO} (\textit{MINimum-size ABductive explanations for linear models with a Reject Option}). MINABRO explicitly accounts for the presence of a reject option in linear models. For accepted predictions, MINABRO employs a greedy algorithm adapted from \cite{marquessilva2020}, which was originally designed for linear classifiers without a rejection mechanism. Our adaptation provably computes minimum-size abductive explanations for linear models with a reject option while preserving the same log-linear time complexity as the original method. For rejected predictions, MINABRO relies on a $0$--$1$ ILP (Integer Linear Programming) formulation to compute minimum-size abductive explanations of rejection. Although this formulation is NP-hard in theory, our experimental results show that it remains computationally efficient in practice and often outperforms the linear-programming-based approach of \cite{rochafilho2023}, which computes abductive explanations but does not guarantee minimum size.


Our experimental evaluation on ten public datasets demonstrates the advantages of MINABRO over two reference methods: Anchors, a sampling-based heuristic, and AbLinRO, a logic-based method that employs linear programming. While both MINABRO and AbLinRO ensure full fidelity to the underlying model, AbLinRO does not guarantee that the resulting explanations are of minimum size. Regarding computational efficiency, MINABRO is orders of magnitude faster than both reference mthods, maintaining stable execution times even in high-dimensional scenarios like MNIST. Notably, even when solving the NP-hard formulation for rejected instances, MINABRO remains more efficient than the linear-programming-based AbLinRO. Furthermore, while the heuristic nature of Anchors leads to shorter but unfaithful explanations, MINABRO provably achieves the smallest possible explanation size among the ones that guarantee full model fidelity. Furthermore, MINABRO consistently produces more concise explanations than AbLinRO, particularly for rejected instances, where explanations tend to be larger. This difference is further illustrated through a visual example of a rejected digit from the MNIST dataset. While both MINABRO and AbLinRO identify features sufficient to justify the rejection, MINABRO's approach isolates a more compact set of pixels.

The remainder of this paper is organized as follows. Section 2 reviews related work on abductive explanations and reject options in machine learning. Section 3 introduces the necessary background on linear classifiers and formal explainability. In Section 4, we present MINABRO, describing our greedy algorithm for accepted instances and our 0--1 ILP formulation for rejected predictions. Section 5 details the experimental evaluation, comparing our approach with reference methods in terms of efficiency and conciseness. Finally, Section 6 concludes the paper and discusses directions for future work.

\section{Related Work}


The field of XAI has experienced significant growth in recent years, driven by the need to make machine learning systems more transparent, auditable, and trustworthy, especially in high-risk applications. In general, explanation methods can be classified into model-agnostic approaches, which construct explanations from local approximations or surrogate models, and model-specific approaches, which directly exploit the internal structure of the classifier. Among the most popular model-agnostic methods are LIME \cite{ribeiro2016lime} and Anchors \cite{ribeiro2018anchors}, which are widely used due to their flexibility and applicability to black-box models. However, these approaches are essentially based on sampling and probabilistic estimates, and therefore do not provide deterministic guarantees of fidelity \cite{ignatiev2019validating}. As a consequence, they may fail to correctly capture complex decision boundaries and are susceptible to adversarial manipulation \cite{slack2020fooling}.

In contrast, methods based on formal logic have been proposed with the goal of providing explanations that are directly grounded in the true behavior of the model. Logic-based XAI approaches have been developed recently to offer explanations with such guarantees for a variety of machine learning models \cite{shi2018,ignatiev2019abduction,marquessilva2020,darwiche2020reasons,izza2021explaining,gorji2022sufficient,marquessilva2022logic,bassan2023towards,audemard23computing,rochafilho2023}. In this context, abductive explanations are minimal subsets of features that support the classifier’s decision under any admissible values of the remaining features. That is, no feature in the explanation can be removed without invalidating the decision.


Despite their desirable properties, computing abductive explanations is, in general, an NP-hard problem \cite{ignatiev2019abduction}, which limits the scalability of these approaches. Computing minimum-
size abductive explanations is an even more challenging problem. To overcome this limitation, recent research \cite{marquessilva2020} has focused on specific classes of models, seeking more efficient algorithms at the cost of reduced generality. In particular, \cite{marquessilva2020} shows that minimum-size
abductive explanations for linear models can
be computed in log-linear time by exploiting both the structure of the model and the bounded domain of the
input attributes. Our work builds upon this line of research by extending minimum-size abductive
explanations to the setting where the linear model includes a reject option, which requires explanation mechanisms that remain faithful to the semantics of abstention.

The reject option has been extensively studied as a mechanism to improve safety by abstaining under uncertainty \cite{cortes2016learning,hendrickx2024machine}. The idea of allowing a classifier to abstain from making a decision was introduced by Chow \cite{chow1970}, who showed that, instead of forcing a potentially incorrect decision, the model can choose to reject instances for which it lacks sufficient confidence, thereby establishing an explicit trade-off between error and abstention.

More recently, several works have investigated the complementary problem of \emph{explaining} rejection decisions. L2loRe \cite{punzi2024l2lore} proposes a dedicated method to explain the reject option, focusing on producing human-understandable explanations of abstentions. Similarly, \cite{artelt2022explaininglvq} studies contrastive explanations for rejections in the context of Learning Vector Quantization classifiers. Other model-agnostic approaches have also been proposed, including local example-based explanations of rejection \cite{artelt2023idontknow}, semifactual explanations designed to be diverse \cite{artelt2022evenif}, and local post-hoc explainers for reject decisions \cite{artelt2022modelagnostic}. In addition, \cite{stradiotti2025learning} addresses a related but distinct setting, in which the system learns to reject \emph{low-quality explanations} based on user feedback, rather than rejecting predictions.

In the domain of linear classifiers, \cite{rochafilho2023} integrated a rejection option into the process of logical explanation for linear SVM classifiers, formulating the problem as instances of linear programming that generates explanations both for accepted decisions and for the rejection class. However, the resulting explanations are only \emph{subset-minimal} and are not guaranteed to have minimum size, which can negatively affect interpretability when explanations remain unnecessarily large. In contrast, our work targets \emph{minimum-size} abductive explanations both for accepted and rejected predictions.



\section{Background}\label{background}

\paragraph{Machine Learning and Binary Classification Problems.} In machine learning, binary classification problems are defined over a set of features $\mathcal{F} = \{f_1, ..., f_n\}$ and a set of two classes $\mathcal{K} = \{-1, +1\}$. In this paper, we consider that each feature $f_i \in \mathcal{F}$ takes values in a bounded real domain
$[l_i, u_i] \subseteq \mathbb{R}$, where $l_i \leq u_i$. An instance or specific data point is a feature assignment
$\mathbf{x} = \{f_1 = x_1, \ldots, f_n = x_n\}$ such that each
$x_i \in [l_i, u_i]$. When convenient, we also represent the same instance as a numerical feature vector $\mathbf{x} = (x_1, \ldots, x_n) \in \mathbb{R}^n$.

A binary classifier $C$ is a function that maps elements in the feature space into the set of classes $\mathcal{K}$. For example, $C$ can map instance $\{f_1 = x_1, f_2 = x_2, ..., f_n = x_n\}$ to class $+1$. Usually, the classifier is obtained by a training process given as input a training set $\{ \mathbf{x}_i,y_i\}^{l}_{i=1}$, where $\mathbf{x}_i \in \mathbb{R}^n$ is an input vector or pattern, $y_i \in \{-1, +1\}$ is the respective class label, and $l$ is the number of input vectors. Then, for each input vector $\mathbf{x}_i$, its input values $x_{i, 1}$, $x_{i, 2}$, ..., $x_{i, n}$ are in the domain of the corresponding features $f_1$, ..., $f_n$. 

\paragraph{Linear Models and Reject Option Classification.}
As a binary classifier, we consider a linear model defined by a given weight vector $\mathbf{w} \in \mathbb{R}^n$ and a given bias term $b \in \mathbb{R}$. Then, the classifier decision is based on the linear score $s(\mathbf{x}) = \mathbf{w} \cdot \mathbf{x} + b$, where $\mathbf{w} \cdot \mathbf{x}$ denotes the inner product between vectors $\mathbf{w}$ and $\mathbf{x}$. The reject option in classification refers to techniques that improve the reliability of decision support systems by abstaining from making a prediction when the model’s output is deemed insufficiently certain \cite{chow1970,hendrickx2024machine}. In binary classification, this typically corresponds to withholding a class
label for instances that lie near the classifier’s uncertainty region,
such as close to the decision boundary in linear models. In practical applications, these abstained instances are often deferred to alternative decision processes, including more complex models or human experts
\cite{el2010foundations}. This behavior is particularly advantageous in high-stakes domains where misclassifications can have severe consequences, making rejection more desirable than risking incorrect predictions. According to~\cite{chow1970}, the optimal classifiers that best handle such a relation can be achieved by the minimization of the empirical risk $\hat{R} = E + w_r R$, where $R$ is the ratio rejected instances, $E$ is the ratio of misclassified instances without including those ones rejected, and $w_r$ is a weight denoting the cost of rejection. A lower $w_r$ gives room for a decreasing error rate at the cost of a higher quantity of rejected instances, with the opposite happening for a higher $w_r$.

Several formulations of the reject option have been proposed in the literature \cite{cortes2016learning,geifman2019selectivenet,mesquita2016classification}. For example, reject rules are often implemented via simple confidence criteria,
particularly in settings where classifiers do not provide calibrated probabilistic outputs \cite{mesquita2016classification}. For linear models, straightforward rejection techniques are commonly based on the distance of instances to the separating hyperplane \cite{mesquita2016classification,rochafilho2023}. If the distance is below a predefined threshold, the instance is withheld from classification, defining a rejection region around the boundary. Such thresholds are determined after model training and encode a desired confidence level in the classifier’s outputs. Therefore, applying this strategy to a linear classifier $C$ leads to the following prediction cases:
\begin{equation}\label{reject_predicition_output}
    C(\mathbf{x}) = 
    \begin{cases}
        \begin{aligned}
         &+1, \quad \text{if } s(\mathbf{x}) > t_+,\\
         &-1, \quad \text{if } s(\mathbf{x}) < t_-,\\
         &\hspace{0.45cm} 0, \quad \text{otherwise},
        \end{aligned}    
    \end{cases} 
\end{equation}
where $t_+$ and $t_-$, with $t^{-} < t^{+}$, are the thresholds for the positive class and negative class, respectively, and 0 is the rejection class. Furthermore, these thresholds are chosen to generate the optimal reject region, which corresponds to the region that minimizes the empirical risk by producing both the ratio of misclassified patterns $E$ and the ratio of rejected patterns $R$. This formulation allows the classifier to abstain from deciding in regions of uncertainty, avoiding potentially inaccurate decisions when the linear score lies between
the thresholds.

\paragraph{Abductive Explanations.} As previously mentioned, most methods for explaining ML models are heuristic, resulting in explanations that can not be fully trusted. Logic-based explainability provides a rigorous alternative to heuristic methods by offering formal guarantees of correctness and irredundancy. Following the framework established by \cite{ignatiev2019abduction}, an \textit{abductive explanation} (AXp) identifies a minimal subset of input features that is sufficient to justify a model's prediction, regardless of the values taken by the remaining features. The condition of subset-minimality ensures that an AXp is irredundant: if any feature is removed from the explanation, the remaining set no longer guarantees the prediction. Formally, we define an AXp as follows:

\begin{definition}[Abductive Explanation \cite{ignatiev2019abduction}]
Let $\mathbf{x} = \{f_1 = x_1, ..., f_n = x_n\}$ be an instance and $C$ be a classifier such that $C(\mathbf{x}) \in \mathcal{K}$. An \emph{abductive explanation} $E$ is a minimal subset of $\mathbf{x}$ such that for all $x_1' \in [l_1, u_1], ..., x_n' \in [l_n, u_n]$, if $x'_j = x_j$ for each $f_j = x_j \in E$, then $C(\{f_1 = x'_1, ..., f_n = x'_n\}) = C(\mathbf{x})$.
\end{definition}

An abductive explanation is a minimal set of features from an instance that are sufficient for the prediction. This subset ensures that, when the values of these features are fixed and the other features are varied within their possible ranges, the prediction remains the same. In other words, it identifies key features that are sufficient for the output. Minimality ensures that the explanation $\mathcal{X}$ does not include any redundant features. In other words, removing any feature $x_j = v_j$ from $\mathcal{X}$ would result in a subset that no longer guarantees the same prediction $c$ over the defined bounds.

While an AXp guarantees that no redundant features are included, an instance may still have multiple distinct AXps of different sizes. Different subsets of features may independently satisfy the conditions for an AXp. This distinction leads to the definition of a \textit{minimum-size abductive explanation}, which seeks the most concise justification possible:

\begin{definition}[Minimum-Size Abductive Explanation] Let $\mathbf{x} = \{f_1 = x_1, ..., f_n = x_n\}$ be an instance and $C$ be a classifier such that $C(\mathbf{x}) \in \mathcal{K}$. An abductive explanation $E$ for $\mathbf{x}$ is a \emph{minimum-size abductive explanation} if, for every other abductive explanation $E'$, we have $|E| \le |E'|$. 
\end{definition}

The distinction between AXps and minimum-size AXps is significant from both a computational and a practical perspective. While any AXp ensures the absence of redundant features, finding a minimum-size AXp is generally a more challenging optimization problem, as it requires searching among all possible AXps for one with the smallest cardinality.

\section{MINABRO}

In this section, we present MINABRO, our method for computing
minimum-size abductive explanations for linear classifiers equipped with a
reject option. We cover the two possible outcomes of such classifiers:
\emph{rejection} and \emph{classification}. Accordingly, we first describe how
MINABRO computes explanations for rejected instances. We then describe the classified case.

\subsection{Explaining Rejected Instances}\label{minabro_rejected}

In this section we show how our method computes minimum-size abductive explanations for linear models with a reject option. We begin with the case of rejected instances. 

\begin{definition}[Minimum-size abductive explanation of rejection]
\label{def:min_abd_reject}
Let $\mathbf{x}$ be an instance and $s$ be a score function of a linear classifier such that $t_- < s(\mathbf{x}) < t_+$, i.e., $\mathbf{x}$ is rejected.
A set of feature assignments $E \subseteq \mathbf{x}$ is an
\emph{abductive explanation of rejection} for $\mathbf{x}$ if, for all
$x'_1 \in [l_1,u_1],\ldots,x'_n \in [l_n,u_n]$, whenever $x'_j=x_j$ for every
$f_j=x_j \in E$, it holds that:
\[
t_- < s(\{f_1=x'_1,\ldots,f_n=x'_n\}) < t_+.
\]
Moreover, $E$ is a \emph{minimum-size} abductive explanation of rejection if it
has minimum cardinality among all abductive explanations of rejection for
$\mathbf{x}$.
\end{definition}

Let $\mathbf{x}$ be an instance and let $E \subseteq \mathbf{x}$ be any
candidate explanation. If a feature assignment $f_i = x_i$ is not included in $E$, then all values in its domain $[l_i, u_i]$ must be considered when checking whether the score remains within the rejection region $t_- < s(\{f_1 = x'_1, ..., f_n = x'_n\}) < t_+$. Given any subset $E$ of $\mathbf{x}$, the score $s(E \cup \{f_i = x'_i \mid f_i = x_i \not\in E\})$ can be decomposed into the sum of a fixed contribution, associated with the features in $E$, and a free contribution, associated with the features not in $E$, as follows:
\begin{equation}\label{decomposition_score}
s(E \cup \{f_i = x'_i \mid f_i = x_i \not\in E\}) =
\underbrace{\sum_{f_j = x_j \in E} w_j x_j + b}_{\mathrm{fixed}}
+
\sum_{f_j = x_j \not\in E} w_j x'_j.
\end{equation}

Therefore, for a given subset $E$, verifying whether $t_- < s(E \cup \{f_i = x'_i \mid f_i = x_i \not\in E\}) < t_+$ reduces to analyzing the maximum and minimum values that the free term $\sum_{f_j = x_j \not\in E} w_j x'_j$ can attain under the domain constraints $x'_j \in [l_j,u_j]$.

If the maximum value of $\sum_{f_j = x_j \not\in E} w_j x'_j$ leads to a score $s(E \cup \{f_i = x'_i \mid f_i = x_i \not\in E\})$ greater than or equal to $t_+$, then $E$ cannot be an abductive explanation of rejection. Similarly, if the minimum value
of $\sum_{f_j = x_j \not\in E} w_j x'_j$ leads to a score less than or equal to $t_-$, then $E$ cannot be an abductive explanation. Since both extrema depend on the choice of $E$, this motivates the following definitions:
\[
s_{\max}(E) = \sum_{f_j = x_j \in E} w_j x_j + b + \underset{\{x'_j \mid f_j=x_j \not\in E\}}{\max} \sum_{f_j = x_j \not\in E} w_j x'_j
\]

\[
s_{\min}(E) = \sum_{f_j = x_j \in E} w_j x_j + b + \underset{\{x'_j \mid f_j=x_j \not\in E\}}{\min} \sum_{f_j = x_j \not\in E} w_j x'_j
\]

Since each feature $f_i$ ranges over a bounded domain $[l_i,u_i]$, both
extrema can be computed in closed form. To maximize $\sum_{f_j = x_j \notin E} w_j x'_j$, we independently choose, for each free feature, the value that maximizes its individual contribution. In particular, if $w_j \ge 0$ then the contribution is maximized by setting $x'_j = u_j$, whereas if $w_j < 0$ then it is maximized by setting $x'_j = l_j$. Similarly, to minimize the sum, we set $x'_j = l_j$ when $w_j \ge 0$, and $x'_j = u_j$ when $w_j < 0$. Features with $w_j = 0$ do not affect the score and can be ignored in these extrema. Therefore, we obtain:
\begin{equation}\label{score_max}
s_{\max}(E) = \sum_{f_j = x_j \in E} w_j x_j + b +
\sum_{f_j = x_j \not\in E}
\begin{cases}
w_j \, u_j, & \text{if } w_j \ge 0, \\
w_j \, l_j, & \text{if } w_j < 0.
\end{cases}
\end{equation}

\begin{equation}\label{score_min}
s_{\min}(E) = \sum_{f_j = x_j \in E} w_j x_j + b +
\sum_{f_j = x_j \not\in E}
\begin{cases}
w_j \, l_j, & \text{if } w_j \ge 0, \\
w_j \, u_j, & \text{if } w_j < 0.
\end{cases}
\end{equation}

Thus, a subset $E \subseteq \mathbf{x}$ is an abductive explanation of rejection
if and only if it satisfies both constraints $s_{\max}(E) < t_{+}$ and $s_{\min}(E) > t_{-}$. Then, to compute a minimum-size explanation, we must find an $E$ of minimum cardinality satisfying these inequalities. We solve this task using a $0$--$1$ ILP formulation.

We introduce binary decision variables $z_j \in \{0,1\}$ for each feature
$f_j \in \mathcal{F}$, where $z_j = 1$ iff $f_j = x_j \in E$. The objective is to minimize the explanation size:
\[
\min \sum_{j=1}^{n} z_j.
\]

To express the constraints $s_{\max}(E) < t_+$ and $s_{\min}(E) < t_-$ in terms of the variables $z_j$, we define, for each feature $f_j$, its worst-case contribution when the feature is \emph{not} fixed by the explanation, and its actual contribution when it is fixed to the observed value $x_j$. Specifically, we define:
\[
\alpha_j^{\max} =
\begin{cases}
w_j u_j, & \text{if } w_j \ge 0,\\
w_j l_j, & \text{if } w_j < 0,
\end{cases}
\qquad
\alpha_j^{\min} =
\begin{cases}
w_j l_j, & \text{if } w_j \ge 0,\\
w_j u_j, & \text{if } w_j < 0,
\end{cases}
\]
and
\[
\beta_j = w_j x_j.
\]

Intuitively, $\alpha_j^{\max}$ is the contribution of feature $f_j$ to the
largest possible score when $f_j$ is allowed to vary freely in its domain,
whereas $\alpha_j^{\min}$ is the contribution to the smallest possible score. Then, $\sum_{j=1}^n \alpha_j^{\max}$ and $\sum_{j=1}^n \alpha_j^{\min}$ correspond to the worst-case score bounds when no feature is fixed. In contrast, $\beta_j$ is the contribution of $f_j$ when it is fixed to its observed value in the instance $\mathbf{x}$.

Using these quantities, we can rewrite the worst-case score bounds for any
subset $E$ encoded by $(z_1,\ldots,z_n)$. We start from the fully worst-case scenario, where \emph{no} feature is fixed, which yields the constants $b + \sum_{j=1}^{n}\alpha_j^{\max}$ and $b + \sum_{j=1}^{n}\alpha_j^{\min}$. Then, whenever $z_j = 1$, feature $f_j$ is included in the explanation, and its
worst-case contribution $\alpha_j^{\max}$ (or $\alpha_j^{\min}$) is replaced by its actual contribution $\beta_j$. This replacement can be expressed by adding a correction term $(\beta_j - \alpha_j^{\max})$ (resp.\ $(\beta_j - \alpha_j^{\min})$) to the corresponding bound. Therefore, for any subset $E$ encoded by $(z_1, z_2, ..., z_n)$, we obtain:
\[
s_{\max}(E)
=
b + \sum_{j=1}^{n} \alpha_j^{\max}
+
\sum_{j=1}^{n} z_j(\beta_j - \alpha_j^{\max}),
\]
\[
s_{\min}(E)
=
b + \sum_{j=1}^{n} \alpha_j^{\min}
+
\sum_{j=1}^{n} z_j(\beta_j - \alpha_j^{\min}).
\]

The role of the binary variables becomes explicit in these expressions.
If $z_j=0$, the correction term vanishes and the model retains the worst-case
contribution for $f_j$. If $z_j=1$, the correction term is activated, canceling the worst-case term $\alpha_j^{\max}$ (or $\alpha_j^{\min}$) and replacing it with the observed contribution $\beta_j$. The ideia is that we start from a baseline worst-case score in which \emph{no feature is fixed}, i.e., $E=\emptyset$. In this case, every feature is allowed to vary, and therefore each feature $f_j$ contributes its worst-case value $\alpha_j^{\max}$. Hence, the baseline upper bound is $b + \sum_{j=1}^{n} \alpha_j^{\max}$. When we include a feature $f_j=x_j$ in the explanation (i.e., when $z_j=1$), we are forced to replace $\alpha_j^{\max}$ by $\beta_j$ in the upper bound. This replacement is encoded by the correction term $z_j(\beta_j - \alpha_j^{\max})$. If $z_j=1$, the correction becomes
$\beta_j - \alpha_j^{\max}$, meaning that we \emph{subtract} the worst-case
contribution $\alpha_j^{\max}$ and \emph{add} the real contribution $\beta_j$. The term $\alpha_j^{\max}$ appears twice because the first occurrence comes from the baseline worst-case score, while the second occurrence appears inside the correction term to cancel the baseline contribution whenever $z_j=1$.

Finally, since $\mathbf{x}$ is rejected, an abductive explanation must
ensure that the score remains inside the rejection interval under any admissible
assignment of the remaining features. Thus, the explanation must satisfy both:
\[
s_{\max}(E) \le t_{+}
\qquad\text{and}\qquad
s_{\min}(E) \ge t_{-}.
\]
Putting everything together, we obtain the following $0$--$1$ ILP for computing minimum-size abductive explanations for rejection:
\begin{equation}\label{ilp_reject}
\begin{aligned}
\min_{z \in \{0,1\}^n}
\quad &
\sum_{j=1}^{n} z_j
\\[2mm]
\text{s.t.}\quad &
b + \sum_{j=1}^{n} \alpha_j^{\max}
+
\sum_{j=1}^{n} z_j(\beta_j - \alpha_j^{\max})
\le t_+,
\\
&
b + \sum_{j=1}^{n} \alpha_j^{\min}
+
\sum_{j=1}^{n} z_j(\beta_j - \alpha_j^{\min})
\ge t_-.
\end{aligned}
\end{equation}

\begin{theorem}
\label{thm:ilp_opt_reject}
Let $\mathbf{x}$ be an instance such that $t_- < s(\mathbf{x}) < t_+$.
Let $(z_1^\star,\ldots,z_n^\star)$ be an optimal solution of the $0$--$1$ ILP
formulation defined in Equation~\ref{ilp_reject}, and let
$E^\star=\{f_j=x_j \mid z_j^\star=1\}$.
Then $E^\star$ is a minimum-size abductive explanation of rejection for
$\mathbf{x}$, as in Definition~\ref{def:min_abd_reject}.
\end{theorem}

\begin{proof}
First, consider any feasible solution $(z_1,\ldots,z_n)$ of
\eqref{ilp_reject} and let $E=\{f_j=x_j \mid z_j=1\}$ be the induced subset.

By the definitions of $s_{\max}(E)$ and $s_{\min}(E)$, the first constraint in
\eqref{ilp_reject} is equivalent to enforcing $s_{\max}(E)\le t_{+}$, while
the second constraint is equivalent to enforcing $s_{\min}(E)\ge t_{-}$.
Therefore, for every admissible completion of the features in
$\mathbf{x}\setminus E$ within their domains, the score remains in the interval
$[t_-,t_+]$, and thus $E$ is an abductive explanation of rejection.

Conversely, let $E$ be any abductive explanation of rejection for $\mathbf{x}$.
Define $(z_1,\ldots,z_n)$ by setting $z_j=1$ iff $f_j=x_j \in E$.
Since $E$ guarantees that every admissible completion stays inside $[t_-,t_+]$,
it follows that $s_{\max}(E)\le t_{+}$ and $s_{\min}(E)\ge t_{-}$.
Hence, $(z_1,\ldots,z_n)$ satisfies both constraints in \eqref{ilp_reject} and
is feasible.

Finally, observe that $\sum_{j=1}^{n} z_j = |E|$ for the set induced by
$(z_1,\ldots,z_n)$.
Thus, minimizing $\sum_{j=1}^{n} z_j$ is equivalent to minimizing the
cardinality of the explanation.
Therefore, an optimal solution $(z_1^\star,\ldots,z_n^\star)$ induces an
abductive explanation $E^\star$ of minimum size.\qed
\end{proof}

The solution $(z_1^\star, ..., z_n^\star)$ directly induces a mininum-size abductive explanation $E^\star = \{f_j = x_j \mid z^\star_j = 1\}$. By construction, $E^\star$ is sufficient to ensure that every completion of the features in $\mathbf{x} \setminus E^\star$ within their domains keeps the instance inside the reject option region $[t_-, t_+]$, and the objective guarantees that $E^\star$ has minimum cardinality.

\subsection{Explaining Classified Instances}

We now describe how MINABRO computes minimum-size abductive explanations for
instances that are \emph{classified} (i.e., not rejected). We first focus on positively classified instances, and later discuss the negative case, which follows by symmetry.


\begin{definition}[Minimum-size abductive explanation of classification]
\label{def:min_abd_class}
Let $\mathbf{x}$ be an instance and let $s$ be the score function of a linear
classifier with reject option thresholds $t_- < t_+$. Assume that $\mathbf{x}$ is
\emph{classified} as positive prediction, i.e., $s(\mathbf{x}) \ge t_+$. A set of feature assignments $E \subseteq \mathbf{x}$ is an \emph{abductive explanation of classification} for $\mathbf{x}$ if, for all $x'_1 \in [l_1,u_1],\ldots,x'_n \in [l_n,u_n]$, whenever $x'_j=x_j$ for every $f_j=x_j \in E$, the classifier output remains unchanged, i.e.,
\[
s(\{f_1=x'_1,\ldots,f_n=x'_n\}) \ge t_+
\]

Moreover, $E$ is a \emph{minimum-size} abductive explanation of classification if
it has minimum cardinality among all abductive explanations of classification
for $\mathbf{x}$.
\end{definition}

Our approach is similar to the greedy method of~\cite{marquessilva2020} for linear classifiers without a reject option. The main difference is that, in
our setting, explanations must guarantee that the score remains above the
positive rejection threshold $t_{+}$, whereas in~\cite{marquessilva2020} the
corresponding condition is defined with respect to the standard decision
boundary at $0$.

As in the rejected case, for a fixed $E$ the score can be decomposed into a fixed part and a free part, as in Equation~\ref{decomposition_score}. To guarantee that the prediction remains positive under any admissible values of the remaining features, it suffices to enforce the \emph{worst-case lower bound} on the score. Thus, we require $s_{\min}(E) \ge t_{+}$, where $s_{\min}(E)$ is defined as in Equation~\ref{score_min}. To compute a minimum-size explanation efficiently, we rewrite $s_{\min}(E)$ as a baseline term plus additive \emph{gains} obtained by fixing features. Again, we use $\alpha_j^{\min}$ and $\beta_j$ as in Subsection~\ref{minabro_rejected}. If no feature is fixed, the worst-case lower bound becomes $b + \sum_{j=1}^{n} \alpha_j^{\min}$. Whenever $f_j=x_j$ is included in the explanation, the contribution $\alpha_j^{\min}$ is replaced by $\beta_j$, increasing the worst-case lower bound by $\delta_j^{+} = \beta_j - \alpha_j^{\min}$. By construction, $\delta_j^{+} \ge 0$ for all $j$, since $\alpha_j^{\min}$ is the minimum admissible contribution of feature $j$. Therefore, for any subset $E$ we can write:
\[
s_{\min}(E) = b + \sum_{j=1}^{n} \alpha_j^{\min} + \sum_{f_j=x_j \in E} \delta_j^{+}.
\]
Thus, the condition $s_{\min}(E) \ge t_{+}$ is equivalent to:
\[
\sum_{f_j=x_j \in E} \delta_j^{+} \ge t_{+} - (b + \sum_{j=1}^{n} \alpha_j^{\min}).
\]

Then, $t_{+} - (b + \sum_{j=1}^{n} \alpha_j^{\min})$ denotes the \emph{required margin} that must be covered by fixing features. Since each feature has unit cost in the objective, as we require minimum size, and contributes a nonnegative gain $\delta_j^{+}$, a minimum-size explanation can be obtained by selecting features in decreasing order of $\delta_j^{+}$ until the inequality is satisfied.

MINABRO computes a minimum-size explanation for a positive instance by sorting
all features by $\delta_j^{+}$ in nonincreasing order and adding them to $E$
until $s_{\min}(E) \ge t_{+}$ holds. Computing $\delta_j^{+}$ for all features requires linear time, i.e., $O(n)$ where $n$ is the number of features. Scanning the sorted list to determine the smallest feasible prefix also requires $O(n)$ time. Therefore, the overall running time is dominated by the sorting step, yielding a total time of $O(n\log n)$. 

\begin{proposition}
\label{prop:greedy_positive_optimal}
Let $\mathbf{x}$ be an instance such that $s(\mathbf{x}) \ge t_{+}$. Consider the greedy algorithm that starts with $E=\emptyset$ and iteratively
adds a feature $f_j=x_j$ with maximum $\delta_j^{+}$ among the remaining ones,
until $\sum_{f_j=x_j\in E}\delta_j^{+}\ge t_{+} - (b + \sum_{j=1}^{n} \alpha_j^{\min})$ (equivalently, until $s_{\min}(E)\ge t_{+}$). Then, the greedy algorithm returns a minimum-size abductive explanation for the positive prediction.
\end{proposition}

\begin{proof}

Let the greedy algorithm select features in nonincreasing order of gains:
$\delta_{1}^{+}\ge \delta_{2}^{+}\ge \cdots \ge \delta_{n}^{+}$,
and let $E_g$ be the set returned by the algorithm.
By definition, the greedy algorithm returns $E_g$ with $k$ features such that 
\begin{equation}
\begin{aligned}
\sum_{i=1}^{k}\delta_{i}^{+}
&\ge t_{+} - \left(b + \sum_{j=1}^{n} \alpha_j^{\min}\right),
\\
\sum_{i=1}^{k-1}\delta_{i}^{+}
&< t_{+} - \left(b + \sum_{j=1}^{n} \alpha_j^{\min}\right).
\end{aligned}
\label{eq:greedy_prefix}
\end{equation}


We now show that no abductive explanation with fewer than $k$ features exists.
Let $E$ be any subset of features with $|E|=m$.
Since $\delta_{1}^{+},\dots,\delta_{n}^{+}$ are sorted in non-increasing
order, the sum of the $m$ largest gains upper-bounds the sum of gains of any
$m$ features. Hence,
\begin{equation}
\label{eq:topm_dominates}
\sum_{f_j=x_j\in E}\delta_j^{+}
\le
\sum_{i=1}^{m}\delta_{i}^{+}.
\end{equation}

Assume, for contradiction, that there exists an abductive explanation
$E^\star$ with $|E^\star|\le k-1$.
Then $E^\star$ satisfies the constraint $\sum_{f_j=x_j\in E^\star}\delta_j^{+}\ge t_{+} - \left(b + \sum_{j=1}^{n} \alpha_j^{\min}\right)$.
Combining this with~\eqref{eq:topm_dominates} (for $m=|E^\star|\le k-1$) yields
\[
t_{+} - \left(b + \sum_{j=1}^{n} \alpha_j^{\min}\right)
\le
\sum_{f_j=x_j\in E^\star}\delta_j^{+}
\le
\sum_{i=1}^{|E^\star|}\delta_{i}^{+}
\le
\sum_{i=1}^{k-1}\delta_{i}^{+},
\]
which contradicts~\eqref{eq:greedy_prefix}.
Therefore, no explanation with fewer than $k$ features exists. Since $E_g$ has size $k$ and is feasible by construction, it is a minimum-size abductive explanation. \qed
\end{proof}

The construction for negative predictions is analogous. Let $\mathbf{x}$ be an instance such that $s(\mathbf{x}) \le t_{-}$. In this case, an abductive explanation $E$ must ensure that the score remains \emph{below} the rejection threshold under any admissible perturbation of the
unexplained features, i.e., for all $x'_1 \in [l_1,u_1],\ldots,x'_n \in [l_n,u_n]$, if $x'_j=x_j$ for every $f_j=x_j \in E$, then: $s(\{f_1=x'_1,\ldots,f_n=x'_n\}) \le t_{-}$.

Therefore, instead of reasoning about the worst-case \emph{minimum} score as in
the positive case, we consider the worst-case \emph{maximum} score $s_{\max}(E)$.
The explanation condition becomes $s_{\max}(E) \le t_-$. Therefore use
$\alpha_j^{\max}$ instead of $\alpha_j^{\min}$ in the construction. Fixing a feature $f_j$ replaces its worst-case contribution $\alpha_j^{\max}$ by
the observed contribution $\beta_j = w_j x_j$. Hence, we define $\delta_j^{-} = \alpha_j^{\max} - \beta_j$, which measures how much the worst-case maximum score decreases when $f_j$ is included in the explanation.

As before, this yields a greedy procedure where features are selected according
to their ability to decrease the worst-case maximum score, and the same
optimality argument applies (up to symmetry), following the proof structure
of~\cite{marquessilva2020}. The time complexity of the negative case is the same as in the positive case: computing all $\delta_j^{-}$ takes $O(n)$ time, sorting them takes $O(n\log n)$ time, and selecting the smallest feasible prefix takes $O(n)$ time, leading again to an overall complexity of $O(n\log n)$.

\section{Experimental Evaluation}

In this section, we evaluate the performance of MINABRO from two complementary perspectives: (i) computational efficiency and (ii) conciseness of explanations. The main goal is to analyze the trade-off between explanatory rigor and practical feasibility by comparing MINABRO with two reference methods. The first one is Anchors \cite{ribeiro2018anchors}, a widely used heuristic and model-agnostic explainer. Since Anchors does not explicitly account for abstention, we adapt it by modeling rejection as an additional class, and thus considering a three-class classifier with labels in $\{-1, 0, +1\}$. The second one is the method in~\cite{rochafilho2023}, which we refer to as AbLinRO throughout the paper. Importantly, AbLinRO computes abductive explanations but does not necessarily guarantee minimum-size explanations. Importantly, both MINABRO and AbLinRO compute abductive explanations that are formally guaranteed to be correct with respect to the underlying classifier. That is, the explanations returned by these methods are provably sufficient to preserve the model's prediction (or rejection). Therefore, unlike heuristic explainers, their fidelity to the model semantics is guaranteed by construction. For this reason, we do not evaluate explanation fidelity in the experiments, and instead focus on computational efficiency and explanation conciseness.

\subsection{Experimental Setup}

The experiments were conducted on ten public binary classification datasets from the UCI Machine Learning Repository\footnote{https://archive.ics.uci.edu} and OpenML~\cite{vanschoren2014openml}, selected to cover different levels of dimensionality. All features were scaled to the interval $[0,1]$, following a standard preprocessing procedure. In our experiments, we treat this interval as the effective domain of each feature, i.e., we assume $f_i \in [0,1]$ for every attribute.

In our experiments, we use logistic regression as a representative linear classifier, trained using a stratified hold-out strategy. Each classifier was trained on 70\% of the original dataset using the \texttt{liblinear} solver with $l_2$ regularization from \texttt{scikit-learn}. To define the rejection region, we fixed the rejection cost to $w_r = 0.24$ for all datasets (as in \cite{rochafilho2023}), with the exception of the \textit{Credit Card} dataset. For the latter, a cost of $w_r = 0.24$ resulted in no instances from the test set falling into the rejection region. Therefore, we adopted a lower rejection cost of $w_r = 0.04$. The rejection thresholds $t_{-}$ and $t_{+}$ were obtained exclusively from training data by minimizing the empirical risk, as described in Section~\ref{background} and in line with established practices \cite{chow1970,mesquita2016classification}. For each dataset, the calibrated thresholds define a dataset-specific rejection region in the classifier score space, whose width is given by $t_{+} - t_{-}$. This quantity provides a direct measure of the degree of conservatism adopted by the classifier: narrower rejection zones correspond to more selective decision policies, while wider zones indicate a higher tolerance to uncertainty.

For rejected instances, MINABRO computes explanations by solving a single $0$--$1$ ILP optimization problem. This formulation was implemented using \texttt{PuLP}\footnote{https://github.com/coin-or/pulp}, with default solver settings. AbLinRO was also implemented using \texttt{PuLP}; however, it constructs explanations through multiple calls to a linear programming (LP) solver, one per feature, during the explanation construction process. Consequently, AbLinRO relies on solving several LP problems for both classified and rejected instances, whereas MINABRO requires solving a single ILP only for rejected instances.

The datasets used in our experiments are summarized in Table~\ref{tab:datasets_estatisticas}, together with the number of features, the calibrated rejection thresholds, the corresponding rejection width, and the number of instances for which explanations were computed. \textit{Rejection Rate} indicates the percentage of test instances for which the model abstained from making a prediction. \textit{Accuracy w/ RO} refers to the model's performance considering only the instances that were not rejected (i.e., the accuracy on the accepted set). \textit{Accuracy w/o RO} represents the accuracy the model would have achieved on all the test set had it been forced to classify them without a reject option. For most datasets, we computed explanations for all available instances in the test set. However, due to computational constraints imposed by Anchors and AbLinRO, we selected a subset of the test set for the \textit{Credit Card}, \textit{Covertype}, and \textit{MNIST (3 vs 8)} datasets to compute explanations. All experiments were executed on machine running Windows~11, equipped with an Intel Core i5-11400H processor (2.70 GHz) and 12~GB of RAM. The source code to reproduce our experimental results is publicly available\footnote{\url{https://github.com/gleilsonpedro/Explanation_With_Rejection_final}}.

\begin{table}[htbp]
\centering
\caption{Summary of datasets and experimental setup, including the number of instances explained, classifier calibration parameters, and classifier performance metrics.}
\label{tab:datasets_estatisticas}
\resizebox{\textwidth}{!}{%
\setlength{\tabcolsep}{3pt}
\begin{tabular}{lrcrrrrrr}
\toprule
\textbf{Dataset} & 
\textbf{\makecell[c]{\# Explained\\Instances}} & 
\textbf{\makecell[c]{\# Features}} & 
\textbf{\makecell[c]{$t^+$\\}} & 
\textbf{\makecell[c]{$t^-$\\}} & 
\textbf{\makecell[c]{Rejection\\Width}} & 
\textbf{\makecell[c]{Rejection\\Rate (\%)}} & 
\textbf{\makecell[c]{Accuracy w/o\\RO (\%)}} &
\textbf{\makecell[c]{Accuracy\\w/ RO (\%)}} \\
\midrule
Banknote         & 412 & 4   & 0.01 & $-0.35$ & 0.36 & 58.74 & 61.41 & 94.71 \\
Vertebral Col.   & 93  & 6   & 1.10 & $-0.07$ & 1.17 & 37.63 & 76.34 & 86.21 \\
Pima Indians     & 231 & 8   & 0.07 & $-0.01$ & 0.08 & 1.30  & 74.46 & 75.00 \\
Heart Disease    & 90  & 13  & 0.19 & $-0.11$ & 0.30 & 42.22 & 85.56 & 96.15 \\
Credit Card      & 2563 & 29  & 4.47 & $-0.01$ & 4.48 & 0.35  & 99.73& 99.96 \\
Breast Cancer    & 171 & 30  & 0.97 & $-0.01$ & 0.98 & 19.88 & 93.57 & 98.54 \\
Covertype        & 742 & 54  & 0.25 & $-0.19$ & 0.44 & 11.19 & 76.68 & 79.21 \\
Spambase         & 1381& 57  & 0.30 & $-0.48$ & 0.78 & 68.21 & 62.27 & 94.99 \\
Sonar            & 63  & 60  & 0.26 & $-0.01$ & 0.27 & 71.43 & 68.25 & 94.44 \\
MNIST (3 vs 8)   & 41  & 784 & 0.82 & $-1.31$ & 2.13 & 7.32  & 90.24 & 94.74 \\
\bottomrule
\end{tabular}%
}
\vspace{1mm}
\begin{minipage}{\textwidth}
\scriptsize
\end{minipage}
\end{table}

\begin{table}[H]
\centering
\caption{Average execution time and standard deviation (ms) for classified and rejected instances.}
\label{tab:runtime_unified}
\resizebox{\textwidth}{!}{%
\setlength{\tabcolsep}{3pt} 
\begin{tabular}{lcccccc}
\hline
\multirow{2}{*}{\textbf{Dataset}} & \multicolumn{2}{c}{\textbf{MINABRO}} & \multicolumn{2}{c}{\textbf{Anchors}} & \multicolumn{2}{c}{\textbf{AbLinRO}} \\
\cline{2-7}
 & \textbf{Clas.} & \textbf{Rej.} & \textbf{Clas.} & \textbf{Rej.} & \textbf{Clas.} & \textbf{Rej.} \\
\hline
Banknote & 1.38 $\pm$ 0.28 & 1.47 $\pm$ 0.26 & 140.89 $\pm$ 266.85 & 64.97 $\pm$ 87.47 & 167.94 $\pm$ 46.66 & 237.13 $\pm$ 48.21 \\
Vertebral Col. & 1.43 $\pm$ 0.24 & 1.61 $\pm$ 0.26 & 422.66 $\pm$ 582.89 & 159.62 $\pm$ 102.17 & 284.73 $\pm$ 29.44 & 383.43 $\pm$ 56.99 \\
Pima Indians & 1.64 $\pm$ 0.48 & 1.85 $\pm$ 0.08 & 501.58 $\pm$ 447.69 & 1332.36 $\pm$ 889.97 & 345.53 $\pm$ 71.27 & 332.56 $\pm$ 10.83 \\
Heart Disease & 1.64 $\pm$ 0.27 & 1.96 $\pm$ 0.21 & 1099.30 $\pm$ 373.09 & 1174.24 $\pm$ 391.55 & 528.83 $\pm$ 42.57 & 973.86 $\pm$ 128.96 \\
Credit Card & 1.92 $\pm$ 0.40 & 2.80 $\pm$ 0.84 & 189.01 $\pm$ 325.32 & 32682.01 $\pm$ 15838.11 & 1192.66 $\pm$ 142.95 & 1678.53 $\pm$ 371.69 \\
Breast Cancer & 1.60 $\pm$ 0.53 & 1.80 $\pm$ 0.69 & 6526.21 $\pm$ 6557.96 & 5003.79 $\pm$ 1140.48 & 781.94 $\pm$ 56.07 & 1427.27 $\pm$ 107.00 \\
Covertype & 2.25 $\pm$ 0.52 & 4.78 $\pm$ 0.50 & 34569.58 $\pm$ 30987.04 & 67035.24 $\pm$ 48883.12 & 2554.02 $\pm$ 401.04 & 3359.55 $\pm$ 193.80 \\
Spambase & 2.92 $\pm$ 0.75 & 5.51 $\pm$ 1.47 & 5291.85 $\pm$ 48700.48 & 14507.73 $\pm$ 90224.29 & 2413.84 $\pm$ 332.13 & 3785.25 $\pm$ 480.74 \\
Sonar & 3.44 $\pm$ 0.62 & 4.96 $\pm$ 0.55 & 32436.80 $\pm$ 31821.72 & 8098.75 $\pm$ 8592.18 & 2699.58 $\pm$ 319.70 & 4572.22 $\pm$ 719.92 \\
MNIST  & 23.24 $\pm$ 4.92 & 167.26 $\pm$ 12.71 & 257871.79 $\pm$ 312222.84 & 300590.33 $\pm$ 259829.42 & 67574.06 $\pm$ 4091.18 & 64727.33 $\pm$ 4553.26 \\
\hline
\end{tabular}%
}
\end{table}

\subsection{Computational Efficiency}

Computational efficiency is a critical factor for the adoption of minimum-size abductive explanations in real-world scenarios, particularly in interactive or
time-sensitive settings. Table~\ref{tab:runtime_unified} reports the average execution time \emph{to compute explanations}, separating the results for \emph{classified} and \emph{rejected} instances. Classified instances correspond to predictions that fall outside the rejection zone. Rejected instances represent the setting where MINABRO relies on a $0$--$1$ ILP formulation to compute minimum-size abductive explanations of rejection, i.e., by solving an NP-hard optimization problem. Importantly, MINABRO computes explanations with full fidelity to the underlying model and provably minimum size. In contrast, Anchors provides no formal fidelity guarantees, whereas AbLinRO guarantees fidelity but does not guarantee minimum-size explanations.

The results show that MINABRO consistently outperforms all reference methods in terms of execution time, both for classified and rejected instances. Notably, the runtime of MINABRO remains low and stable even for high-dimensional datasets such as \emph{Spambase}, \emph{Sonar}, and \emph{MNIST} compared to the reference methods.

In contrast, the heuristic sampling-based approach Anchors and the
linear-programming-based method AbLinRO exhibit a substantial increase in computational cost as the number of features grows. This behavior highlights the advantage of MINABRO, which exploits the linearity of the classifier and the bounded domains of input features to compute minimum-size explanations efficiently for classified instances. Remarkably, even for rejected instances, where MINABRO relies on a $0$--$1$ ILP formulation and thus solves an NP-hard optimization problem, it remains consistently faster in practice. A plausible explanation is that AbLinRO requires multiple calls to an LP (Linear Programming) solver, one per feature during explanation construction, whereas
MINABRO solves a single $0$--$1$ ILP instance with a compact set of linear constraints.

Overall, these results confirm that MINABRO achieves consistently low execution times, making it particularly suitable for scalable and time-critical
applications.

\begin{table}[htbp]
\centering
\caption{Average explanation size (number of features) for classified and rejected instances.}
\label{tab:explanation_size}
\resizebox{\textwidth}{!}{%
\setlength{\tabcolsep}{4pt} 
\begin{tabular}{lcccccc}
\toprule
\multirow{2}{*}{\textbf{Dataset}} & 
\multicolumn{2}{c}{\textbf{MINABRO}} & 
\multicolumn{2}{c}{\textbf{Anchor}} & 
\multicolumn{2}{c}{\textbf{AbLinRO}} \\
\cmidrule(lr){2-3} \cmidrule(lr){4-5} \cmidrule(lr){6-7}
 & \textbf{Acc.} & \textbf{Rej.} & \textbf{Acc.} & \textbf{Rej.} & \textbf{Acc.} & \textbf{Rej.} \\
\midrule
Banknote         & 2.86 $\pm$ 0.65 & 2.60 $\pm$ 0.70  & 1.37 $\pm$ 0.99 & 1.04 $\pm$ 0.34 & 2.86 $\pm$ 0.65  & 2.73 $\pm$ 0.67 \\
Vertebral Col.   & 4.34 $\pm$ 1.04 & 5.06 $\pm$ 0.23  & 2.03 $\pm$ 1.57 & 1.86 $\pm$ 0.54 & 4.41 $\pm$ 0.97  & 5.06 $\pm$ 0.23 \\
Pima Indians     & 4.17 $\pm$ 1.20 & 8.00 $\pm$ 0.00  & 2.49 $\pm$ 0.97 & 5.33 $\pm$ 1.70 & 4.60 $\pm$ 1.22  & 8.00 $\pm$ 0.00 \\
Heart Dis.       & 7.15 $\pm$ 1.41 & 6.84 $\pm$ 1.27  & 2.63 $\pm$ 0.66 & 2.79 $\pm$ 0.73 & 7.44 $\pm$ 1.46  & 6.61 $\pm$ 1.51 \\
Credit Card      & 15.76 $\pm$ 1.42 & 17.33 $\pm$ 5.10 & 0.19 $\pm$ 0.39 & 7.56 $\pm$ 0.83 & 20.33 $\pm$ 1.56 & 20.44 $\pm$ 3.95 \\
Breast Cancer    & 1.69 $\pm$ 0.46 & 2.00 $\pm$ 0.00  & 2.05 $\pm$ 0.84 & 2.00 $\pm$ 0.00 & 1.69 $\pm$ 0.46  & 2.00 $\pm$ 0.00 \\
Covertype        & 20.06 $\pm$ 5.79 & 43.57 $\pm$ 0.70 & 4.47 $\pm$ 1.87 & 5.41 $\pm$ 2.25 & 22.64 $\pm$ 5.90 & 44.51 $\pm$ 1.08 \\
Spambase         & 28.20 $\pm$ 2.59 & 44.03 $\pm$ 1.53 & 1.24 $\pm$ 2.93 & 2.17 $\pm$ 7.05 & 30.19 $\pm$ 4.26 & 48.89 $\pm$ 1.34 \\
Sonar            & 42.33 $\pm$ 8.43 & 39.33 $\pm$ 4.65 & 2.72 $\pm$ 0.89 & 2.04 $\pm$ 0.29 & 52.44 $\pm$ 3.93 & 48.80 $\pm$ 4.06 \\
MNIST (3 vs 8)   & 253.2 $\pm$ 42.5 & 533.0 $\pm$ 9.3  & 0.16 $\pm$ 1.0  & 0.00 $\pm$ 0.00 & 361.3 $\pm$ 52.0 & 573.3 $\pm$ 7.4 \\
\bottomrule
\end{tabular}%
}
\end{table}

\subsection{Explanation Size}

We evaluate the size of the generated explanations, i.e., the average number of features included in an explanation. For rejected instances, fidelity
requires simultaneously enforcing both lower and upper constraints on the
classifier score. For this reason, we report explanation sizes separately for classified and rejected instances, enabling a more fine-grained comparison across methods. Table~\ref{tab:explanation_size} reports the average explanation size obtained by each method.





A clear gap can be observed between Anchors and the logic-based methods. Anchors produces substantially shorter explanations (e.g., 1.24 features for classified instances on \emph{Spambase}). This behavior reflects a fundamental methodological difference: Anchors is a sampling-based heuristic that does not enforce worst-case guarantees, whereas MINABRO and AbLinRO compute explanations that are fully faithful to the classifier semantics. Therefore, the larger explanation sizes observed for MINABRO when compared to Anchors should not be interpreted as a weakness, but as the natural cost of ensuring full fidelity guarantees.

The results also highlight the limitations of the heuristic approach on high-dimensional or imbalanced datasets. For instance, on \emph{Credit Card} and \emph{MNIST}, Anchors yields mean explanation sizes close to zero (e.g., $0.19$ and $0.16$, respectively). This occurs because Anchors seeks a rule with a precision $\geq 95\%$. In highly imbalanced scenarios (like \emph{Credit Card}), the ``empty rule'' (using no features) often satisfies this precision threshold simply due to the class prior probability. Similarly, in high-dimensional spaces, the sampling strategy struggles to converge to complex rules, frequently defaulting to trivial explanations that fail to capture the true decision boundary. In contrast, MINABRO is robust to these factors, consistently identifying the actual set of features responsible for the model's prediction.

When comparing MINABRO with AbLinRO, the results show that MINABRO typically
matches or improves conciseness. This difference is expected, since AbLinRO does not guarantee minimum size, whereas MINABRO explicitly optimizes for minimum-size explanations in both the classified and rejected settings.

It is also interesting to note that, on average, explanations for rejected
instances tend to be larger than those for classified ones across most datasets. This behavior is likely explained by the fact that explanations of rejection must simultaneously enforce two constraints on the classifier score, namely that it remains above $t_{-}$ and below $t_{+}$. In contrast, explanations for classified instances only need to preserve a single-sided decision constraint.

\subsection{Visual Comparison on MNIST Rejected Instances}

To complement the quantitative evaluation, we provide a visual comparison between the explanations produced by MINABRO and AbLinRO for rejected instances. Figure~\ref{fig:mnist_rejected_example} illustrates a representative example from the MNIST (3 vs.\ 8) task. In this example, the original digit is rejected by the trained logistic regression classifier with reject option. The left image shows the original input instance, while the subsequent images highlight the pixels included in the respective explanations. Red pixels correspond to features selected as part of the abductive explanation.

\begin{figure}[H]
\centering
\includegraphics[width=\linewidth]{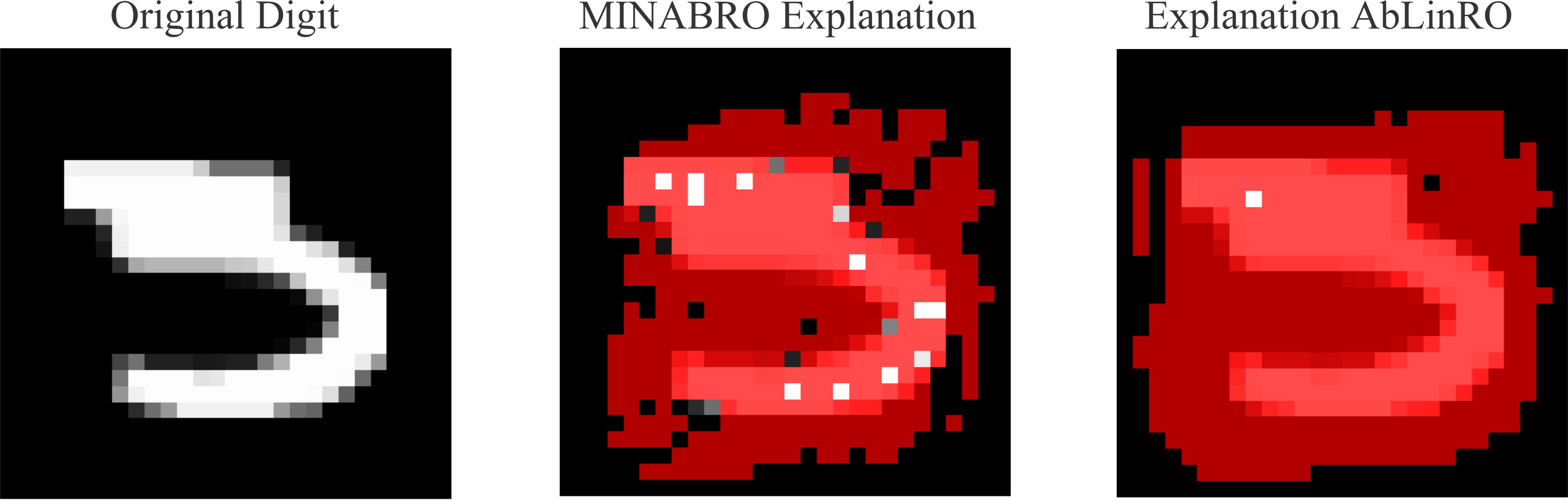}
\caption{Comparison of explanations for a rejected digit in the MNIST (3 vs.\ 8) task.
Red pixels indicate features included in the explanation.
\textbf{Left:} Original digit, which is rejected by the trained logistic regression classifier with reject option. 
\textbf{Center:} MINABRO explanation. 
\textbf{Right:} AbLinRO explanation.}
\label{fig:mnist_rejected_example}
\end{figure}

Both methods generate explanations that are sufficient to guarantee that the
instance remains inside the rejection region. More precisely, the selected pixels form a subset of features such that, if these pixel values are fixed and all other pixels are allowed to vary arbitrarily within their domain, the classifier is still guaranteed to output rejection. However, the explanation produced by AbLinRO covers a substantially larger portion of the image, fixing 495 out of 784 pixels (63.1\%), while MINABRO selects only 417 pixels (53.2\%). This visual comparison highlights the ability of MINABRO to produce more compact explanations for rejected instances. It is important to note that these explanations characterize the decision mechanism of the trained logistic regression with a reject option, rather than human visual intuition about the digit. In particular, the highlighted pixels may lie outside the visible stroke of the digit, reflecting the fact that the linear classifier may rely on background or contextual patterns learned during training.

\section{Conclusions and Future Work}


In this work, we introduced \textbf{MINABRO}, an efficient method for
computing minimum-size abductive explanations for linear classifiers equipped
with a reject option. For instances that are \emph{classified}, MINABRO computes minimum-size explanations efficiently by exploiting the linear structure of the classifier on the bounded feature domains and a greedy procedure inspired by the approach of \cite{marquessilva2020}. For \emph{rejected} instances, where fidelity requires simultaneously enforcing both lower and upper constraints on the classifier
score, MINABRO relies on a $0$--$1$ ILP formulation that optimizes explanation
cardinality under the reject-option semantics.

The experimental results show that MINABRO consistently achieves substantial
runtime improvements over the reference methods across all datasets, including
high-dimensional settings. Importantly, this advantage is maintained even in the
rejected setting, despite the need to solve an NP-hard optimization problem.
At the same time, MINABRO produces explanations that are fully faithful by
construction and, in contrast to existing approaches, guarantees minimum
cardinality. By ensuring the production of minimum-size explanations, our method significantly enhances interpretability, as more compact justifications are easier for human experts to inspect and validate. This advantage is particularly relevant for rejected instances, which inherently tend to yield larger explanations due to the need to satisfy the dual bounding constraints that define the rejection region.

Overall, these results indicate that MINABRO is a practical approach for computing formally correct and minimum-size abductive explanations for linear classifiers with reject option. The ability to provide concise justifications for rejections is of paramount importance in practical applications, as it allows practitioners to diagnose the specific factors leading to model uncertainty, identify potential biases, and determine when human intervention is necessary for safe decision-making. In particular, explaining abstentions contributes to trustworthy AI by making explicit the conditions under which a system deliberately refrains from issuing a prediction, which is a key mechanism for safe AI deployment and responsible AI risk management. Furthermore, by ensuring both minimum-size and correctness of explanations, our method effectively supports high-stakes applications where the reliability and predictability of the provided justifications are essential requirements.





As future work, we plan to extend the proposed framework to multiclass scenarios equipped with a reject option and to investigate the generation of inflated abductive explanations \cite{izza2024delivering,rocha2025generalizing,izza2025most}. The latter allows for more general and informative justifications by replacing specific feature values with maximal ranges that still guarantee the model's decision, thus providing a broader understanding of the classifier. In addition, we intend to explore the use of our approach to provide more reliable model-agnostic explanations by approximating complex non-linear models, such as deep neural networks, with local linear surrogates equipped with a reject option. In this framework, the rejection region of the local linear model can be designed to represent areas where the surrogate cannot faithfully mimic the complex model's behavior or where the underlying model itself exhibits high uncertainty. By applying MINABRO to these local approximations, we can generate minimum-size explanations that formally justify why a prediction is made or, more importantly, provide rigorous reasons for why a decision should be rejected due to lack of local reliability or model confidence.




\bibliographystyle{splncs04}
\bibliography{referencias}
\end{document}